\documentclass[11pt,a4paper,logo]{googledeepmind}

\pdftrailerid{redacted}

\usepackage{subcaption}
\usepackage{placeins}
\usepackage{float}
\usepackage{mathtools}
\usepackage[textsize=tiny]{todonotes}

\definecolor{highlightbg}{RGB}{248,249,250}
\definecolor{keyblue}{RGB}{0,82,155}

\usepackage{arxiv_style}
\setarxivlogo{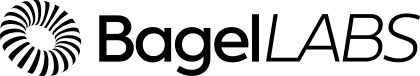}
\setarxivlogowidth{90pt}
\setarxivlogoraise{0pt}
\setarxivdatetext{}

\usepackage[numbers,sort&compress]{natbib}
\hypersetup{plainpages=false}

\theoremstyle{plain}
\newtheorem{theorem}{Theorem}[section]
\newtheorem{proposition}[theorem]{Proposition}

\theoremstyle{definition}
\newtheorem{definition}[theorem]{Definition}

\theoremstyle{remark}
\newtheorem{remark}[theorem]{Remark}

\title{Expert-Data Alignment Governs Generation Quality in Decentralized Diffusion Models}
\correspondingauthor{marcos@bagel.com}

\makeatletter
\renewcommand\AB@authnote[1]{}
\renewcommand\AB@affilnote[1]{}
\makeatother

\author{Marcos Villagra}
\author{Bidhan Roy}
\author{Raihan Seraj}
\author{Zhiying Jiang}

\affil{Bagel Labs}

\begin{abstract}
Decentralized Diffusion Models (DDMs) route denoising through experts trained independently on disjoint data clusters, which can strongly disagree in their predictions. What governs the quality of generations in such systems? We present the first ever systematic investigation of this question. A priori, the expectation is that minimizing denoising trajectory sensitivity---minimizing how perturbations amplify during sampling---should govern generation quality. We demonstrate this hypothesis is incorrect: a stability–quality dissociation. Full ensemble routing, which combines all expert predictions at each step, achieves the most stable sampling dynamics and best numerical convergence while producing the worst generation quality (FID 47.9 vs. 22.6 for sparse Top-2 routing). Instead, we identify expert-data alignment as the governing principle: generation quality depends on routing inputs to experts whose training distribution covers the current denoising state. Across two distinct DDM systems, we validate expert-data alignment using (i) data-cluster distance analysis, confirming sparse routing selects experts with data clusters closest to the current denoising state, and (ii) per-expert analysis, showing selected experts produce more accurate predictions than non-selected ones, and (iii) expert disagreement analysis, showing quality degrades when experts disagree. For DDM deployment, our findings establish that routing should prioritize expert-data alignment over numerical stability metrics.
\end{abstract}

\begin{document}
\maketitle

%%%%%%%%%%%%%%%%%%%%%%%%%%%%%%%%%%%%%%%%%%%%%%%%%%%%%%%%%%%%%%%%%%%%%%%%%%%%%%%
%%%%%%%%%%%%%%%%%%%%%%%%%%%%%%%%%%%%%%%%%%%%%%%%%%%%%%%%%%%%%%%%%%%%%%%%%%%%%%%
%%%%%%%%%%%%%%%%%%%%%%%%%%%%%%%%%%%%%%%%%%%%%%%%%%%%%%%%%%%%%%%%%%%%%%%%%%%%%%%
\section{Introduction}

Decentralized Diffusion Models (DDMs) \cite{mcallister2025ddm} combine
independently trained diffusion experts \cite{ho2020denoising} via an inference-time router. Because
experts are trained on disjoint data clusters and can strongly disagree in their predictions,
understanding what governs generation quality becomes crucial—yet this question has not been systematically studied.

A natural hypothesis is that numerical stability determines quality \cite{yang2024lipschitz}, that is, routing
strategies that minimize trajectory sensitivity should produce superior samples.
We demonstrate that this hypothesis is incorrect. In this work, we observe that full ensemble routing, in which all expert predictions are combined, achieves
the lowest trajectory sensitivity
and best numerical convergence, yet
produces the worst generation quality. These results rule out trajectory sensitivity as the
primary determinant of generation quality.

We identify \emph{expert-data alignment} (routing inputs to experts trained
on similar data) as the governing principle. When sparse routing (e.g., Top-2)
selects experts whose training distribution covers the current denoising state, each expert
produces coherent velocity predictions that combine meaningfully. Full ensemble
routing, by contrast, forces all experts to process every input; since each
expert is trained on only a subset of the data, most of them are processing
out-of-distribution data at any given time. The averaged velocity field may be
smooth, but it points toward an incoherent compromise rather than the data
manifold.

We provide direct experimental validation of this principle across two distinct DDM systems. Data-cluster distance analysis confirms that sparse routing selects experts with data clusters closest to the input embedding. Per-expert prediction quality analysis
shows that selected experts produce velocity predictions with higher alignment
to the blended output. Expert disagreement analysis demonstrates that
disagreement under full ensemble correlates with quality degradation.

Although numerical stability does not govern quality, understanding when DDM
sampling converges remains valuable. Classical stability analysis suggests DDMs
should fail: Lipschitz constants of deep networks grow exponentially with depth
\cite{fazlyab2019efficient,virmaux2018lipschitz,yang2024lipschitz}, and Gr\"onwall's inequality
implies small perturbations amplify over integration \cite{hairer1993solving}.
We also examine \emph{trajectory-local sensitivity}, denoted $\widehat{L}_{\text{eff}}^{(h)}$,
which formalizes the idea that Jacobian
spectral norms remain bounded along realized sampling paths. Empirically, we observe that trajectories exhibit moderate
sensitivity compared to
worst-case global bounds. While $\widehat{L}_{\text{eff}}^{(h)}$ does not predict quality across routing strategies, it may serve as a within-strategy diagnostic for identifying numerically sensitive samples.

The contributions of this work are as follows.
\begin{itemize}
  \item \textbf{Expert-data alignment principle.} We identify expert-data
        alignment as the
        primary determinant of generation quality in DDMs. We provide direct
        experimental validation via (i) data-cluster distance analysis showing sparse routing selects experts with data clusters closest to the input,
        (ii) per-expert analysis demonstrating that
        selected experts produce superior velocity predictions, and (iii) expert disagreement analysis showing quality degrades when experts disagree.
  \item \textbf{Stability--quality dissociation.} We demonstrate that trajectory
        sensitivity does not govern generation quality: full ensemble routing achieves
        the lowest $\widehat{L}_{\text{eff}}^{(h)}$ and step-refinement disagreement, yet
        produces the worst Fr\'echet Inception Distance (FID) \cite{heusel2017gans}. This rules out numerical stability as the primary
        quality determinant. We additionally explore trajectory-local sensitivity as a within-strategy diagnostic, finding weak predictive power.
\end{itemize}

This paper is organized as follows.
Section~\ref{sec:related-work} reviews related work.
Section~\ref{sec:background} introduces notation and background on DDMs.
Section~\ref{sec:dissociation} establishes the stability--quality dissociation.
Section~\ref{sec:alignment-experiments} identifies expert-data alignment as the
governing principle and provides direct experimental validation.
Section~\ref{sec:sensitivity-analysis} presents trajectory sensitivity analysis and
empirical validation.
Sections~\ref{sec:discussion} and~\ref{sec:conclusions} discuss implications
and conclude.

%%%%%%%%%%%%%%%%%%%%%%%%%%%%%%%%%%%%%%%%%%%%%%%%%%%%%%%%%%%%%%%%%%%%%%%%%%%%%%%
%%%%%%%%%%%%%%%%%%%%%%%%%%%%%%%%%%%%%%%%%%%%%%%%%%%%%%%%%%%%%%%%%%%%%%%%%%%%%%%
%%%%%%%%%%%%%%%%%%%%%%%%%%%%%%%%%%%%%%%%%%%%%%%%%%%%%%%%%%%%%%%%%%%%%%%%%%%%%%%
\section{Related Work}\label{sec:related-work}

\paragraph{Diffusion as ODE/SDE and numerical stability.}
Diffusion sampling involves Lipschitz-constrained ODEs where discretization error affects accuracy.
Sampling can be expressed as a probability-flow Ordinary Differential Equation (ODE) \cite{song2021scorebased}, where Lipschitz constants
and discretization error determine solver accuracy.
Recent work by \citet{tan2025stork} addresses temporal stiffness via a specialized solver called STORK
based on stabilized Runge-Kutta methods for stiff diffusion ODEs. STORK's
stabilized solvers could, in principle, be combined with our approach (using
stable solvers for temporal stiffness while using sparse routing to control spatial
sensitivity). We use the released Euler and Heun solvers to isolate the effect of routing
strategies, but STORK-style solvers may provide additional gains in decentralized settings.

\paragraph{Local Lipschitz analysis.}
Local Lipschitz bounds have been extensively studied in the neural network literature.
\citet{jordan2020exactly} developed methods for exactly computing local Lipschitz constants, enabling input-specific
sensitivity analysis rather than global worst-case bounds.
Our work applies this trajectory-local perspective to diffusion sampling.
The key novelty is not the local Lipschitz concept itself, but
its application to understanding decentralized diffusion dynamics and the discovery
that routing implicitly stabilizes the sampling dynamics.

\paragraph{Decentralized diffusion.}
DDMs show that decentralized experts, when routed, can match a monolithic
diffusion objective \cite{mcallister2025ddm}. Our work investigates why
such combination succeeds despite expert disagreement, providing a
stability-based explanation complementary to the original capacity arguments.

\paragraph{DDM vs.\ traditional Mixture-of-Experts.}
DDMs are \emph{ensembles of independently trained models}, not
traditional Mixture-of-Experts (MoE). In standard MoE architectures
experts are FFN layers within a shared backbone, trained jointly with load
balancing losses, and routed at the token level \cite{shazeer2017outrageously,fedus2021switch}. In DDM, each ``expert'' is
a \emph{complete diffusion model} trained in isolation on a disjoint data
partition (no shared parameters, no gradient communication, no joint
training). Routing occurs at the \emph{input level} (entire noisy images)
rather than token level, and experts are combined only at inference time.
Concurrent work applies traditional MoE architectures \emph{within} diffusion models \cite{fei2024ditmoe,sun2025ecdit,yang2025diffmoe,wang2025expertrace,cheng2025diffmoe,liu2025dsmoe,zheng2025dense2moe}; DDM instead combines complete, independently trained models.
This distinction matters for stability analysis: DDM experts can produce
arbitrarily different outputs for the same input (having never coordinated
during training), whereas MoE experts share a representational backbone
that constrains their disagreement. Our analysis specifically addresses
the stability challenges arising from DDM's decentralized training.

\paragraph{Data-aware routing and expert specialization.}
The importance of matching inputs to appropriately trained experts is
well-established in MoE systems. Sparsely-gated MoE architectures rely on
learned routing to direct inputs to relevant experts \cite{shazeer2017outrageously},
with subsequent work analyzing expert utilization and load balancing
\cite{dai2022stablemoe,zhou2022expertchoice}. In federated learning,
data heterogeneity across clients creates analogous challenges: models
trained on non-IID partitions may produce poor predictions on out-of-distribution
inputs \cite{li2020federated}. Our work provides direct experimental evidence
that this principle governs sample quality in DDMs: sparse routing succeeds
precisely because it maintains alignment between inputs and expert training
distributions.

%%%%%%%%%%%%%%%%%%%%%%%%%%%%%%%%%%%%%%%%%%%%%%%%%%%%%%%%%%%%%%%%%%%%%%%%%%%%%%%
%%%%%%%%%%%%%%%%%%%%%%%%%%%%%%%%%%%%%%%%%%%%%%%%%%%%%%%%%%%%%%%%%%%%%%%%%%%%%%%
%%%%%%%%%%%%%%%%%%%%%%%%%%%%%%%%%%%%%%%%%%%%%%%%%%%%%%%%%%%%%%%%%%%%%%%%%%%%%%%
\section{Background}\label{sec:background}

%%%%%%%%%%%%%%%%%%%%%%%%%%%%%%%%%%%%%%%%%%%%%%%%%%%%%%%%%%%%%%%%%%%%%%%%%%%%%%%
%%%%%%%%%%%%%%%%%%%%%%%%%%%%%%%%%%%%%%%%%%%%%%%%%%%%%%%%%%%%%%%%%%%%%%%%%%%%%%%
%%%%%%%%%%%%%%%%%%%%%%%%%%%%%%%%%%%%%%%%%%%%%%%%%%%%%%%%%%%%%%%%%%%%%%%%%%%%%%%
\subsection{Notation}\label{sec:notation}
For a differentiable scalar function $f:\mathbb{R}^d\times[0,1]\to\mathbb{R}$,
we let $\nabla_x f(x,t)\in\mathbb{R}^d$ denote the gradient with respect to $x$.
We use $J_x f(x,t)\in\mathbb{R}^{d\times d}$ to denote the Jacobian matrix.
We use $\|\cdot\|$ for the matrix spectral norm and to denote the $\ell_2$-norm over $\mathbb{R}^d$.

%%%%%%%%%%%%%%%%%%%%%%%%%%%%%%%%%%%%%%%%%%%%%%%%%%%%%%%%%%%%%%%%%%%%%%%%%%%%%%%
%%%%%%%%%%%%%%%%%%%%%%%%%%%%%%%%%%%%%%%%%%%%%%%%%%%%%%%%%%%%%%%%%%%%%%%%%%%%%%%
%%%%%%%%%%%%%%%%%%%%%%%%%%%%%%%%%%%%%%%%%%%%%%%%%%%%%%%%%%%%%%%%%%%%%%%%%%%%%%%
\subsection{Diffusion sampling}
Diffusion sampling generates data by integrating an ordinary differential equation (ODE)
$\frac{dx_t}{dt} = v_t(x_t)$ from noise to data \cite{lipman2023flow}.
We call $v = \{v_t\}_{t=0}^1$ a \emph{flow} and the solution $\{x_t\}$ a
\emph{trajectory}. The ODE has a unique solution when $v$ is
Lipschitz continuous \cite{coddington1955theory}.

%%%%%%%%%%%%%%%%%%%%%%%%%%%%%%%%%%%%%%%%%%%%%%%%%%%%%%%%%%%%%%%%%%%%%%%%%%%%%%%
%%%%%%%%%%%%%%%%%%%%%%%%%%%%%%%%%%%%%%%%%%%%%%%%%%%%%%%%%%%%%%%%%%%%%%%%%%%%%%%
%%%%%%%%%%%%%%%%%%%%%%%%%%%%%%%%%%%%%%%%%%%%%%%%%%%%%%%%%%%%%%%%%%%%%%%%%%%%%%%
\subsection{Decentralized Diffusion}
Decentralized Diffusion Models (DDMs) \cite{mcallister2025ddm} train $K$
expert diffusion models in complete isolation on disjoint data partitions.
Unlike traditional MoE (which routes tokens to
different FFN layers within a shared backbone), DDM routes entire inputs
to separate full models at each denoising step.

At inference, a lightweight router predicts weights $w_t^{(k)}(x_t)\geq 0$ with $\sum_{k=1}^K w_t^{(k)}(x_t) = 1$
at each denoising step. The routed velocity field is given by
\begin{equation}
  v_t(x_t) = \sum_{k=1}^K w_t^{(k)}(x_t)\, v_t^{(k)}(x_t).
  \label{eq:routing-intro}
\end{equation}
Sampling integrates $\frac{d x_t}{dt} = v_t(x_t)$ from noise $x_0 \sim \mathcal{N}(0,I)$.

The router outputs probabilities $p_t(k|x_t)$ for each expert $k$.
Full ensemble mode uses $w_t^{(k)}(x_t) = p_t(k|x_t)$ for all $k$.
Top-$k$ routing selects the $k$ experts with highest probability and renormalizes their weights \cite{shazeer2017outrageously}. Top-1 selects a single expert: $v_{\text{Top-1}}(x_t)=v_{k^*}(x_t)$ with $k^*=\arg\max_k p_t(k|x_t)$. This yields a piecewise-smooth vector field with non-differentiability on measure-zero switching surfaces.

In this work, when we say that the DDM router converges we mean numerical convergence of the sampler,
formalized below.

\begin{definition}[DDM sampling convergence]
\label{def:ddm-convergence}
Fix trained experts and a router with a routed flow $v=\{v_t\}_{t=0}^1$.
Let $x_1 \sim \mathcal{N}(0,I)$ be an initial condition at time $t=1$ and let $x_t$ denote the ODE solution at
time $t<1$ (note that $t$ is decreasing). Let $\tilde{x}^{(h)}_t$ denote the output of a numerical ODE
solver with step size $h$,
coupled to the same initial noise $x_1$. We say the DDM sampler \emph{converges in probability} if for
every $\varepsilon>0$, we have
$
P \big(\|\tilde{x}^{(h)}_0 - x_0\| > \varepsilon\big)\to 0 \quad \text{as } h\to 0,
$
where the probability is over the initial noise $x_1$. This notion implies convergence in distribution of
$\tilde{x}^{(h)}_0$ to $x_0$.
\end{definition}

%%%%%%%%%%%%%%%%%%%%%%%%%%%%%%%%%%%%%%%%%%%%%%%%%%%%%%%%%%%%%%%%%%%%%%%%%%%%%%%
%%%%%%%%%%%%%%%%%%%%%%%%%%%%%%%%%%%%%%%%%%%%%%%%%%%%%%%%%%%%%%%%%%%%%%%%%%%%%%%
%%%%%%%%%%%%%%%%%%%%%%%%%%%%%%%%%%%%%%%%%%%%%%%%%%%%%%%%%%%%%%%%%%%%%%%%%%%%%%%
\section{The Stability--Quality Dissociation}
\label{sec:dissociation}

A natural hypothesis is that numerical stability governs generation quality in DDMs, that is,
routing strategies that minimize trajectory sensitivity should produce superior
samples. We establish that this hypothesis is incorrect.

Sections \ref{sec:alignment-experiments} and \ref{sec:sensitivity-analysis} will demonstrate the dissociation:
full ensemble routing achieves the
lowest trajectory sensitivity  and lowest
step-refinement disagreement, yet produces the worst
generation quality---even worse generation quality is observed in \cite{jiang2025paris}.
Top-2 achieves the best generation quality despite
higher trajectory sensitivity. This rules out numerical stability as the quality determinant.

%%%%%%%%%%%%%%%%%%%%%%%%%%%%%%%%%%%%%%%%%%%%%%%%%%%%%%%%%%%%%%%%%%%%%%%%%%%%%%%
%%%%%%%%%%%%%%%%%%%%%%%%%%%%%%%%%%%%%%%%%%%%%%%%%%%%%%%%%%%%%%%%%%%%%%%%%%%%%%%
%%%%%%%%%%%%%%%%%%%%%%%%%%%%%%%%%%%%%%%%%%%%%%%%%%%%%%%%%%%%%%%%%%%%%%%%%%%%%%%
\section{Expert-Data Alignment}
\label{sec:alignment-experiments}

The preceding section ruled out numerical stability as the quality determinant. We now establish \emph{expert-data alignment}---routing inputs to experts trained on similar data---as the governing principle and provide direct experimental validation.

%%%%%%%%%%%%%%%%%%%%%%%%%%%%%%%%%%%%%%%%%%%%%%%%%%%%%%%%%%%%%%%%%%%%%%%%%%%%%%%
\subsection{The Expert-Data Alignment Hypothesis}
\label{sec:alignment-hypothesis}

Full ensemble averaging produces a smoother velocity field because:
(1) averaging over all experts reduces variance in velocity predictions,
(2) this variance reduction directly lowers $\|J_x v(x_t, t)\|$,
and (3) smoother trajectories also exhibit lower step-refinement disagreement.
However, this smoothing forces experts to process out-of-distribution data.
Each expert is trained on only small part of the data (one cluster); when all
experts contribute to every input, most of them process data outside their training
distribution. Therefore, the averaged velocity field may be smooth but it points toward
an incoherent compromise rather than the data manifold.

Top-2 routing selects the two experts whose training data most closely matches
the current input, keeping each expert producing coherent velocity predictions
that combine meaningfully. The dominant factor for sample quality is whether
experts process data similar to their training distribution (what we call
\emph{expert-data alignment}) rather than minimizing trajectory sensitivity.

If expert-data alignment governs quality, we expect that
\begin{enumerate}
  \item sparse routing should achieve higher alignment (selected experts have lower
cluster distance) than full ensemble;
  \item selected experts should produce superior velocity predictions compared to
non-selected experts; and,
  \item expert disagreement should correlate with quality degradation under full
ensemble.
\end{enumerate}

%%%%%%%%%%%%%%%%%%%%%%%%%%%%%%%%%%%%%%%%%%%%%%%%%%%%%%%%%%%%%%%%%%%%%%%%%%%%%%%
\subsection{Cluster Distance Analysis}
\label{sec:exp-cluster}

We use the pretrained DDM Paris model \cite{jiang2025paris} via released pretrained checkpoints.
The model consists of $K=8$ experts trained on a subset of LAION-Aesthetics \cite{schuhmann2022laion}.
The dataset was partitioned into 8 semantic
clusters via two-stage hierarchical k-means on DINOv2-ViT-L/14 embeddings. The model
uses a DiT-B/2 router ($\sim$129M parameters) and 8 DiT-XL/2 experts
(a modified version of the DiT-XL/2 experts with $\sim$606M parameters each, $\sim$5B total).
The router was trained
\emph{post-hoc} on the full dataset, effectively learning to route inputs to
the expert trained on the most similar data. We will use Paris DDM to test the Expert-Data Alignment Hypothesis.

Let $\mathcal{C}_k$ denote the training data cluster for expert $k$, and let
$d(x, \mathcal{C}_k)$ denote the Euclidean distance from the DINOv2 embedding
of input $x$ to the centroid of cluster $\mathcal{C}_k$. We define
\emph{high expert-data alignment} as the condition where the selected experts
have low $d(x, \mathcal{C}_k)$ relative to non-selected experts.

We test whether sparse routing selects experts whose training clusters match
the input distribution. For $n=500$ samples, we extract DINOv2-ViT-L/14 embeddings at timesteps
$t \in \{0.3, 0.5, 0.7\}$ during sampling. For each state $(x_t, t)$, we compute:
(1) the Euclidean distance from the embedding to each of the 8 cluster centroids
used during expert training;
(2) the experts selected by each routing strategy at that timestep;
(3) the rank of the selected expert(s)' cluster(s) among all 8 clusters, ordered
by distance.

We emphasize that this embedding distance is used only for relative expert ranking;
all conclusions depend solely on rank comparisons, not metric fidelity. Moreover,
the training clusters themselves were defined via k-means in DINOv2 space,
making this the canonical embedding for cluster proximity.

We report two metrics. (i)  the average rank
(1 = closest, 8 = farthest) of the selected expert's training cluster refered to as \emph{Mean cluster rank}, and
(ii)  the percentage of timesteps where at
least one selected expert's cluster is among the two closest to the current input refered to as \emph{Top-2 Match Rate}.
Table~\ref{tab:cluster-distance} shows the results.

\begin{table}[t]
  \caption{\textbf{Cluster distance analysis validates expert-data alignment.}
  Sparse routing selects experts whose training clusters match the input.
  Lower mean rank indicates better alignment (rank 1 = closest). Results averaged over $n=500$
  samples at $t \in \{0.3, 0.5, 0.7\}$. We also include the results for full ensemble as
  a baseline reference.}
  \label{tab:cluster-distance}
  \centering
  \begin{small}
  \begin{tabular}{@{}lcc@{}}
    \toprule
    \textbf{Routing} & \textbf{Mean Cluster Rank} $\downarrow$ & \textbf{Top-2 Match Rate} $\uparrow$ \\
    \midrule
    Top-1 & 1.54 $\pm$ 0.28 & 90.2\% \\
    Top-2 & 1.96 $\pm$ 0.26 & 83.9\% \\
    Full (8) & 4.50 $\pm$ 0.00 & 25.0\% \\
    \bottomrule
  \end{tabular}
  \end{small}
\end{table}

Top-1 and Top-2 achieve mean cluster ranks of 1.54 and 1.96, far below the 4.5 random baseline,
with Top-2 match rates exceeding 83\%.

%%%%%%%%%%%%%%%%%%%%%%%%%%%%%%%%%%%%%%%%%%%%%%%%%%%%%%%%%%%%%%%%%%%%%%%%%%%%%%%
\subsection{Per-Expert Prediction Quality}
\label{sec:exp-expert-quality}
Now we want to test whether selected experts produce superior velocity predictions compared to non-selected experts.

For $n=200$ samples generated with Top-2 routing, we record at each timestep
(1) the blended velocity $v_{t}(x_t) = \sum_{k \in \mathcal{S}} w_t^{(k)} v_t^{(k)}(x_t)$,
(2) the individual velocity predictions $v_t^{(k)}(x_t)$ for all 8 experts, and
(3) the routing weights and selected set $\mathcal{S}$.

For each expert $k$ at each timestep, we compute the velocity alignment score define as
\[
a^{(k)}(x_t) = \frac{v_t^{(k)}(x_t)^\top \cdot v_{t}(x_t)}{\|v_t^{(k)}(x_t)\| \cdot \|v_{t}(x_t)\|},
\]
where $v_t^{(k)}(x_t)^\top$ denotes the transpose of $v_t^{(k)}(x_t)$.
This cosine similarity measures how well expert $k$'s prediction aligns with
the routed velocity used for successful generation.
We report angular deviation $\theta_k = \arccos(a_k)$ in degrees for interpretability.
See Table~\ref{tab:expert-quality} for the results.

\begin{table}[t]
  \caption{\textbf{Selected experts produce better-aligned predictions.}
  Angular deviation (degrees) from blended velocity for selected versus
  non-selected experts under Top-2 routing. Smaller is better. Results
  averaged over $n=200$ samples. Statistical significance via independent samples t-test.}
  \label{tab:expert-quality}
  \centering
  \begin{small}
  \begin{tabular}{@{}lcc@{}}
    \toprule
    \textbf{Expert Status} & \textbf{Angular Dev.} $\downarrow$ & \textbf{Std Dev} \\
    \midrule
    Selected (Top-2) & 3.6° & $\pm$1° \\
    Non-selected & 5.1° & $\pm$1° \\
    \midrule
    Reduction & 1.5° (29\%) & $p < 0.001$ \\
    \bottomrule
  \end{tabular}
  \end{small}
\end{table}

Selected experts achieve smaller angular deviation from the blended velocity (3.6° vs.\ 5.1°; independent samples t-test, $p < 0.001$),
a 29\% reduction confirming systematic identification of coherent experts. The statistical significance
demonstrates that routing systematically identifies the most coherent experts rather than selecting
arbitrarily.

%%%%%%%%%%%%%%%%%%%%%%%%%%%%%%%%%%%%%%%%%%%%%%%%%%%%%%%%%%%%%%%%%%%%%%%%%%%%%%%
\subsection{Expert Disagreement and Sample Quality}
\label{sec:exp-disagreement}
In this section, we want to test whether expert disagreement explains the poor quality of full ensemble routing.

For $n=500$ samples generated with full ensemble, we measure
(1) mean pairwise expert disagreement
$D(x_t) = \frac{1}{\binom{K}{2}} \sum_{i < j} \|v_i(x_t) - v_j(x_t)\|_2$,
(2) trajectory-integrated disagreement: $D_{\text{int}} = \int_0^1 D(x_t) \, dt$, and
(3) perceptual quality as LPIPS distance \cite{zhang2018unreasonable} to the corresponding Top-2 output
(matched initial noise). See Figure~\ref{fig:disagreement-quality} for the results.

\begin{figure}[t]
  \centering
  \includegraphics[scale=1.3]{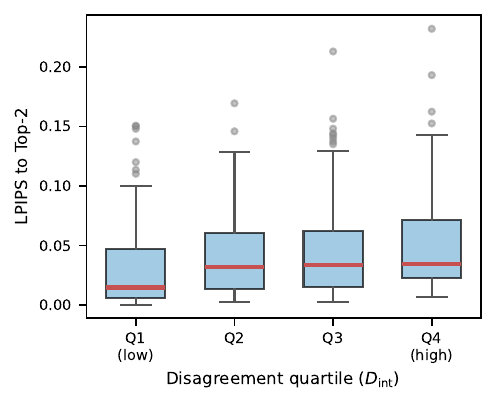}
  \caption{\textbf{Higher expert disagreement degrades perceptual quality in full ensemble routing.}
  Samples binned by trajectory-integrated disagreement (Q1=lowest, Q4=highest).
  Higher disagreement quartiles show greater perceptual distance (LPIPS) from the Top-2 reference,
  explaining why full ensemble underperforms sparse routing.}
  \label{fig:disagreement-quality}
\end{figure}

The monotonic increase in LPIPS across disagreement quartiles confirms
that expert disagreement drives quality degradation in full ensemble
routing.

%%%%%%%%%%%%%%%%%%%%%%%%%%%%%%%%%%%%%%%%%%%%%%%%%%%%%%%%%%%%%%%%%%%%%%%%%%%%%%%
\subsection{MNIST Validation of Expert-Data Alignment}
\label{sec:mnist}

We validate the expert-data alignment hypothesis on a separate MNIST-based DDM with 10 UNet experts and a CNN router.
This controlled setting provides independent confirmation of our main findings on a simpler domain with more specialized experts.

The MNIST DDM consists of 10 independently trained UNet experts (each $\sim$10M parameters) and a lightweight CNN router.
Each expert was trained on a digit-specific subset of MNIST, creating strong expert specialization.
The router was trained separately to predict which expert best matches each input.

We run two experiments mirroring the Paris DDM analysis:
(1) per-expert prediction quality comparing selected vs.\ non-selected experts, and
(2) expert disagreement correlation with quality degradation.
All experiments use the Heun solver with 50 steps and $n=500$ samples per configuration.

\begin{table}[t]
  \caption{\textbf{Selected experts produce better-aligned predictions.}
  Angular deviation (degrees) from blended velocity for selected versus
  non-selected experts under Top-2 routing. Smaller is better. Results
  averaged over $n=500$ samples.}
  \label{tab:mnist-quality}
  \centering
  \begin{small}
  \begin{tabular}{@{}lcc@{}}
    \toprule
    \textbf{Expert Status} & \textbf{Angular Dev.} $\downarrow$ & \textbf{Std Dev} \\
    \midrule
    Selected (Top-2) & 6.4° & $\pm$1° \\
    Non-selected & 11.3° & $\pm$1° \\
    \midrule
    Reduction & 4.9° (43\%) & $p \approx 0$ \\
    \bottomrule
  \end{tabular}
  \end{small}
\end{table}

Table~\ref{tab:mnist-quality} shows that selected experts produce
smaller angular deviation from the blended velocity than non-selected experts.
Selected experts achieve smaller angular deviation from the blended velocity (6.4° vs.\ 11.3°, $p \approx 0$),
a 43\% reduction confirming systematic identification of coherent experts.
The angular deviation gap (4.9°) is substantially larger than Paris DDM (1.5°),
reflecting the stronger specialization of MNIST experts trained on digit-specific subsets.

Under full ensemble routing, we measure the correlation between trajectory-integrated expert disagreement
and output quality degradation (MSE and LPIPS distance to Top-2 reference outputs).
MNIST exhibits a substantially stronger disagreement-quality correlation than Paris DDM.
This stronger correlation confirms that expert disagreement is a robust predictor of quality degradation,
with the effect amplified in settings with stronger expert specialization.

%%%%%%%%%%%%%%%%%%%%%%%%%%%%%%%%%%%%%%%%%%%%%%%%%%%%%%%%%%%%%%%%%%%%%%%%%%%%%%%
\subsection{Summary of Results}

The experiments provide evidence for expert-data alignment as
the governing principle of sample quality in DDMs. Cluster distance analysis
confirms that sparse routing selects experts whose training data matches the
input (Table~\ref{tab:cluster-distance}). Per-expert prediction quality shows
that selected experts produce superior velocity predictions
(Table~\ref{tab:expert-quality}). Expert disagreement analysis explains why
full ensemble fails, that is, averaging across disagreeing experts produces incoherent
velocity fields that point off-manifold (Figure~\ref{fig:disagreement-quality}).
Critically, this correlation provides causal evidence: high disagreement represents
states where alignment naturally breaks down (multiple experts are forced to process
inputs outside their training distribution), and we observe that this misalignment
degrades actual sample quality (LPIPS), not merely step-refinement error.

We also provide independent validation of these findings on a MNIST-based DDM
with 10 specialized expert, where the alignment effects are even
more pronounced.

%%%%%%%%%%%%%%%%%%%%%%%%%%%%%%%%%%%%%%%%%%%%%%%%%%%%%%%%%%%%%%%%%%%%%%%%%%%%%%%
%%%%%%%%%%%%%%%%%%%%%%%%%%%%%%%%%%%%%%%%%%%%%%%%%%%%%%%%%%%%%%%%%%%%%%%%%%%%%%%
%%%%%%%%%%%%%%%%%%%%%%%%%%%%%%%%%%%%%%%%%%%%%%%%%%%%%%%%%%%%%%%%%%%%%%%%%%%%%%%
\section{Trajectory Sensitivity Analysis}
\label{sec:sensitivity-analysis}

Having established that expert-data alignment governs sample quality, we now
analyze trajectory sensitivity to understand when numerical convergence holds
and to develop within-strategy diagnostics. We present a conditional convergence
argument linking trajectory-local sensitivity to DDM convergence, followed by
empirical validation.

%%%%%%%%%%%%%%%%%%%%%%%%%%%%%%%%%%%%%%%%%%%%%%%%%%%%%%%%%%%%%%%%%%%%%%%%%%%%%%%
\subsection{Convergence in Probability under Trajectory-Local Sensitivity}
\label{sec:problem}

We present a conditional argument where sensitivity bounds around a trajectory imply
convergence in probability of the DDM sampler.

\begin{definition}[Trajectory-local sensitivity]\label{def:trajectory-local-sensitivity}
Given a flow $v$ and an initial condition $x_1$, let $\{x_t\}_{t\in[0,1]}$ denote the (exact)
ODE solution for $t\in[0,1]$.
If the solution does not exist over $[0,1]$, we let $L_{\text{eff}}(x_1)=+\infty$. Otherwise,
we define the \emph{effective Lipschitz constant} at $x_1$ as
$
L_{\text{eff}}(x_1) = \sup_{t\in[0,1]} \|J_x v_t(x_t)\|.
$
Here, $\|\cdot\|$ denotes the operator norm induced by the Euclidean norm (i.e., the Jacobian spectral norm).
We call the trajectory \emph{locally stable} if $L_{\text{eff}}(x_1) < \infty$.
\end{definition}

Definition~\ref{def:trajectory-local-sensitivity} introduces $L_{\text{eff}}(x_1)$ as a functional of the exact ODE solution $\{x_t\}$ induced by $(v,x_1)$.
This is not a dynamical claim about where trajectories go, and we do not use it to prove that trajectories enter or remain in low-sensitivity regions.
Instead, boundedness of $L_{\text{eff}}$ is an \emph{assumption} in our numerical convergence argument, which we support empirically via
a numerical approximate proxy $\widehat{L}_{\text{eff}}^{(h)}$ defined below.

\begin{definition}\label{def:empirical-leff}
Given a numerical solver with step size $h$ producing a discrete trajectory $\{\tilde{x}^{(h)}_{t_n}\}_{n=0}^N$,
we define the \emph{empirical effective Lipschitz constant} as
$
\widehat{L}_{\text{eff}}^{(h)}(x_1) := \max_{n\in\{0,\dots,N\}} \|J_x v(\tilde{x}^{(h)}_{t_n},t_n)\|.
$
\end{definition}

The numerical approximation $\widehat{L}_{\text{eff}}^{(h)}$ is a good estimator of $L_{\text{eff}}$ as we
refine the solver (see Appendix \ref{app:approx} for details), and hence,
it can be used as a post-hoc diagnostic.

\begin{remark}[Circularity of the diagnostic]\label{rem:circularity}
Computing $\widehat{L}_{\text{eff}}^{(h)}$ requires the numerical trajectory, which is only available \emph{after} sampling completes. This makes $\widehat{L}_{\text{eff}}^{(h)}$ a retrospective diagnostic rather than an a priori predictor: one cannot use it to decide whether a given initial noise will yield a well-converged sample without first running the sampler. We view this as inherent to trajectory-local analysis and use $\widehat{L}_{\text{eff}}^{(h)}$ primarily for understanding sensitivity differences across routing strategies, not for sample-level prediction.
\end{remark}

\begin{definition}[Sampler sensitivity]\label{def:sampler-sensitivity}
  A sampler is $(L,\delta)$\emph{-trajectory-locally sensitive} if $P_{x_1\sim q}[L_{\text{eff}}(x_1)\le L]\ge 1-\delta$,
  for some noise distribution $q$, initial condition $x_1$, and constants $L$ and $\delta$.
\end{definition}

Consider a one-step numerical ODE solver with step size $h$ (e.g.,
Heun / explicit trapezoidal rule) approximating the initial value
problem $\frac{dx_{t}}{dt} = v_{t}(x_{t})$
on $t \in [0,1]$, with discrete time steps $t_n = 1 - nh$ for $n=0, \dots, N$, where $h=1/N$ and $N$ is the total number of steps. At each step $n$,
the method has a (one-step) local truncation error $\epsilon_n^{\text{local}}$
defined as
$\epsilon_n^{\text{local}}(t_n) = x_{t_{n+1}} - \text{Step}_h(x_{t_n}, t_n)$,
where $\text{Step}_h(x, t)$ denotes the numerical update.
Let $e_N := \|x_{t_N} - \tilde{x}^{(h)}_{t_N}\|$ denote the global solution error at the final step.
A standard Gr\"onwall-based argument shows that on the event $\{L_{\text{eff}}(x_1) \le L\}$,
the global error satisfies $e_N \le C h^p e^{L}$ for constants $C, p > 0$
depending on the solver.
Conditioning on this event and choosing $h$ sufficiently small yields, for any $\varepsilon > 0$,
$
  P\big(e_N > \varepsilon\big)  \le  P\big(L_{\text{eff}}(x_1) > L\big) \le \delta.
$
Full convergence in probability ($\lim_{h\to 0} P(e_N > \varepsilon) = 0$ for all $\varepsilon > 0$)
follows by taking $L \to \infty$ along quantiles of $L_{\text{eff}}$ and
adjusting $h$ accordingly; see Appendices~\ref{app:convergence-proof} for details.

Definitions~\ref{def:trajectory-local-sensitivity}--\ref{def:sampler-sensitivity}
are stated in terms of the exact ODE solution $\{x_t\}$, which is not directly accessible in practice.
Accordingly, our experiments report $\widehat{L}_{\text{eff}}^{(h)}$ (Definition~\ref{def:empirical-leff})
computed along numerical trajectories as a proxy.
This is evidence for sensitivity and it is not, by itself, a guarantee that $L_{\text{eff}}$ is bounded for the exact flow.
All convergence statements below are therefore conditional on the (unobserved)
typical-set of trajectories whose effective Lipschitz constant stays bounded.

%%%%%%%%%%%%%%%%%%%%%%%%%%%%%%%%%%%%%%%%%%%%%%%%%%%%%%%%%%%%%%%%%%%%%%%%%%%%%%%
\subsection{Empirical Validation of the Convergence Argument}
\label{sec:methodology}

We design three experiments to validate the convergence argument, each targeting
one step of the logical chain: (1) bounded local error, (2) bounded error
propagation via trajectory-local sensitivity, and (3) convergence under step-size
refinement. Appendices \ref{app:alignment-details} and \ref{app:sensitivity-details}
provide additional technical details.

%%%%%%%%%%%%%%%%%%%%%%%%%%%%%%%%%%%%%%%%%%%%%%%%%%%%%%%%%%%%%%%%%%%%%%%%%%%%%%%
\subsubsection{Experiment 1: Verifying Bounded Local Discretization Error}
\label{sec:exp-local-error}

The goal of this experiment is to test that the numerical solver introduces bounded local truncation
error at each step.

We measure local truncation error by comparing single-step outputs between Heun's method and a
high-precision reference (Heun with $h/10$ subdivisions).
We compute this at 1000 randomly selected trajectory points across 1,000 samples for each routing strategy.

The following measurements are computed: 1) mean and maximum local error $\epsilon^{\text{local}}$
per routing strategy, 2) scaling with step size at $h \in \{0.02, 0.01, 0.005\}$ for each routing strategy, and
3) distribution of local errors across timesteps and samples for each routing strategy.
Detailed results are reported in Appendix~\ref{app:experiment1}.

%%%%%%%%%%%%%%%%%%%%%%%%%%%%%%%%%%%%%%%%%%%%%%%%%%%%%%%%%%%%%%%%%%%%%%%%%%%%%%%
\subsubsection{Experiment 2: Measuring Trajectory-Local Sensitivity}
\label{sec:exp-leff}
The goal of this experiment is to test that the trajectory exhibits bounded Jacobian spectral norms,
and that this bounds error propagation.
We emphasize that spectral norms quantify worst-case local sensitivity; we do not estimate Jacobian eigenvalues or logarithmic norms along trajectories, and therefore do not directly characterize linear stability of the flow. We treat this as a limitation and leave eigenvalue-based analyses to future work.

For each numerical sampling trajectory $\{\tilde{x}^{(h)}_{t_n}\}_{n=0}^N$ starting from $x_1$
we compute $\widehat{L}_{\text{eff}}^{(h)}(x_1)$. However, computing a Jacobian spectral norm can be time
consuming for latent spaces of high dimension, and hence, we estimate it using the power method.
We repeat the power iteration until the relative change between iterations 9 and 10 is less than $0.5\%$.
For samples with slower convergence, we extend up to 20 iterations.

The following measurements are computed:
1) $\widehat{L}_{\text{eff}}^{(h)}$ distribution per routing strategy including mean and standard deviation,
and 2) temporal profile of $\|J_x v(\tilde{x}^{(h)}_{t_n}, t_n)\|_2$ as a function of $t_n$.
Results are reported in Appendix~\ref{app:experiment2} and Figure~\ref{fig:jacobian-trace}.

\begin{figure}[t]
  \centering
  \includegraphics[scale=0.7]{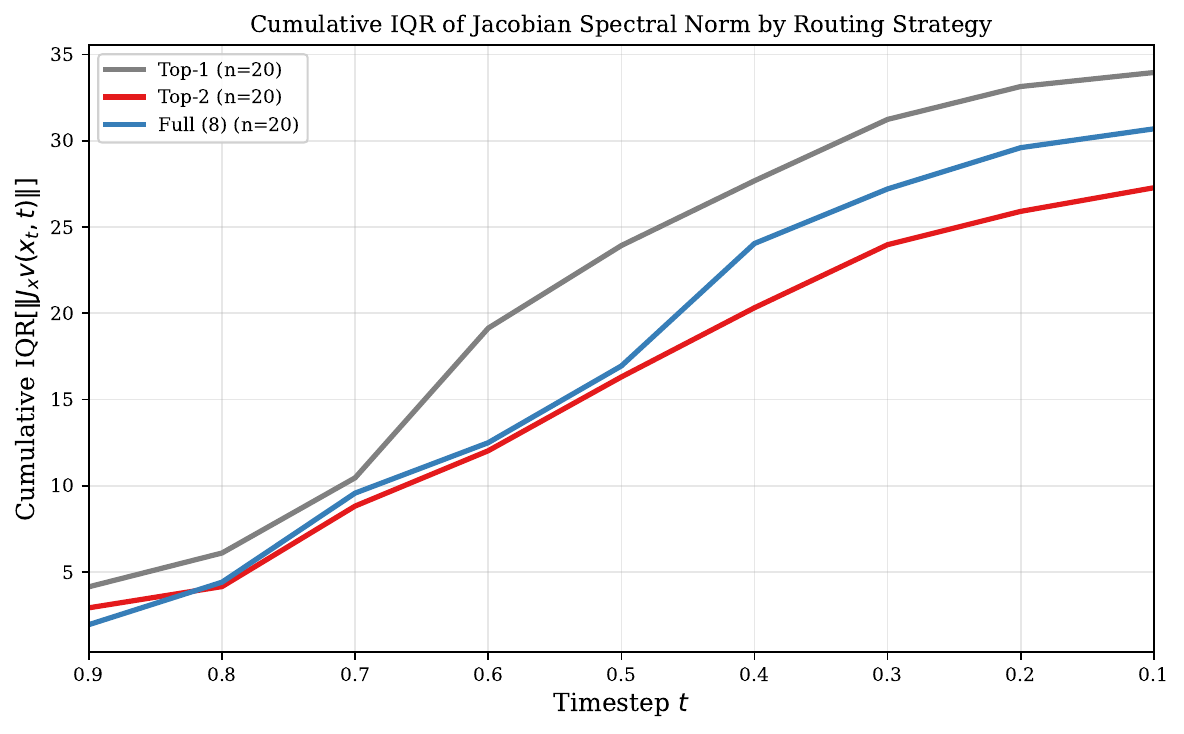}
  \caption{Cumulative IQR of the Jacobian spectral norm $\|\nabla_x v(x_t, t)\|$ as a measure of variability across sampling trajectories.
  The gap between Top-2 and other strategies widens as denoising progresses.
  Mid-trajectory timesteps ($t \in [0.1, 0.9]$); $n{=}20$ samples per strategy.}
  \label{fig:jacobian-trace}
\end{figure}

\begin{table}[t]
  \caption{\textbf{Routing strategy comparison.}
  FID from \cite{jiang2025paris}; $\widehat{L}_{\text{eff}}^{(h)}$ and $\Delta_{\text{refine}}$ measured on $n{=}1000$ samples.
  Full ensemble achieves the lowest trajectory sensitivity and step-refinement disagreement yet produces the worst FID.}
  \label{tab:dissociation}
  \begin{center}
    \begin{small}
      \renewcommand{\arraystretch}{1.1}
      \begin{tabular}{@{}lccc@{}}
        \toprule
        \textbf{Strategy} & \textbf{FID} $\downarrow$ & $\boldsymbol{\widehat{L}_{\textbf{eff}}^{(h)}}$ & $\boldsymbol{\Delta_{\textbf{refine}}}$ $\downarrow$ \\
        \midrule
        \textit{Monolithic} & \textit{29.64} & -- & -- \\
        \midrule
        Top-1 & 30.60 & 18.81 {\scriptsize$\pm$7.15} & 0.075 {\scriptsize$\pm$0.107} \\
        \rowcolor{highlightbg}
        Top-2 & 22.60 & 17.48 {\scriptsize$\pm$6.07} & 0.051 {\scriptsize$\pm$0.070} \\
        Full (8) & 47.89 & 17.07 {\scriptsize$\pm$6.33} & 0.020 {\scriptsize$\pm$0.020} \\
        \bottomrule
      \end{tabular}
    \end{small}
  \end{center}
\end{table}

%%%%%%%%%%%%%%%%%%%%%%%%%%%%%%%%%%%%%%%%%%%%%%%%%%%%%%%%%%%%%%%%%%%%%%%%%%%%%%%
\subsubsection{Experiment 3: Step-Size Refinement and Convergence}
\label{sec:exp-refinement}

The goal of this experiment is to test whether step-size refinement reduces endpoint error, and whether
step-refinement disagreement relates to trajectory-local stability.

For each initial noise $x_1$, we generate samples at $N$ steps (step size $h = 1/N$) and $2N$ steps (step size $h/2$).
We measure the step-refinement disagreement defined as
\begin{equation}\label{eq:delta-refine}
\Delta_{\text{refine}}(x_1) := \text{LPIPS}\big(D(\tilde{x}_0^{(h)}), D(\tilde{x}_0^{(h/2)})\big),
\end{equation}
where $D(\cdot)$ denotes VAE decoding, LPIPS  measures perceptual distance in decoded image space \cite{zhang2018unreasonable},
and $x_1$ is the initial noise.
The metric $\Delta_{\text{refine}}$ is an a posteriori refinement-based error indicator that does not involve Jacobian computation, providing a measure of numerical integration error that avoids methodological coupling with $L_{\text{eff}}$.

The following measurements are computed:
1) $\Delta_{\text{refine}}$ distribution per routing strategy,
2) Spearman correlation coefficient $\rho(\widehat{L}_{\text{eff}}^{(h)}, \Delta_{\text{refine}})$, and
3) area under the curve (AUC) for predicting high-$\Delta_{\text{refine}}$
events using $\widehat{L}_{\text{eff}}^{(h)}$ to measure predictive performance.
Results are reported in Figure~\ref{fig:leff-vs-refine}; see Appendix~\ref{app:experiment3} for methodology details.

%%%%%%%%%%%%%%%%%%%%%%%%%%%%%%%%%%%%%%%%%%%%%%%%%%%%%%%%%%%%%%%%%%%%%%%%%%%%%%%
\subsection{Summary of Results}
Table~\ref{tab:dissociation} shows a summary of the observed results.
Across all routing strategies, correlation between $\widehat{L}_{\text{eff}}^{(h)}$
and step-refinement disagreement $\Delta_{\text{refine}}$
is low ($\rho < 0.08$; Figure~\ref{fig:leff-vs-refine})\footnote{Preliminary analysis of the router
Jacobian term $|\sum_k v_k \nabla_x w_k|$ also shows similarly low correlation
with $\Delta_{\text{refine}}$ (Spearman $\rho = -0.07$, $p = 0.62$ for Top-2 routing, $n{=}50$),
reinforcing that Jacobian-based sensitivity metrics are not aligned with discretization error.}.
This further supports our finding that numerical stability metrics do not govern generation quality, aligning with the observed dissociation between sensitivity and FID.
Factors beyond worst-case Jacobian norms---such as directional alignment of perturbations with the
flow or cancellation effects across timesteps---likely govern actual error accumulation.

\begin{figure}[t]
  \centering
  \includegraphics[width=\columnwidth]{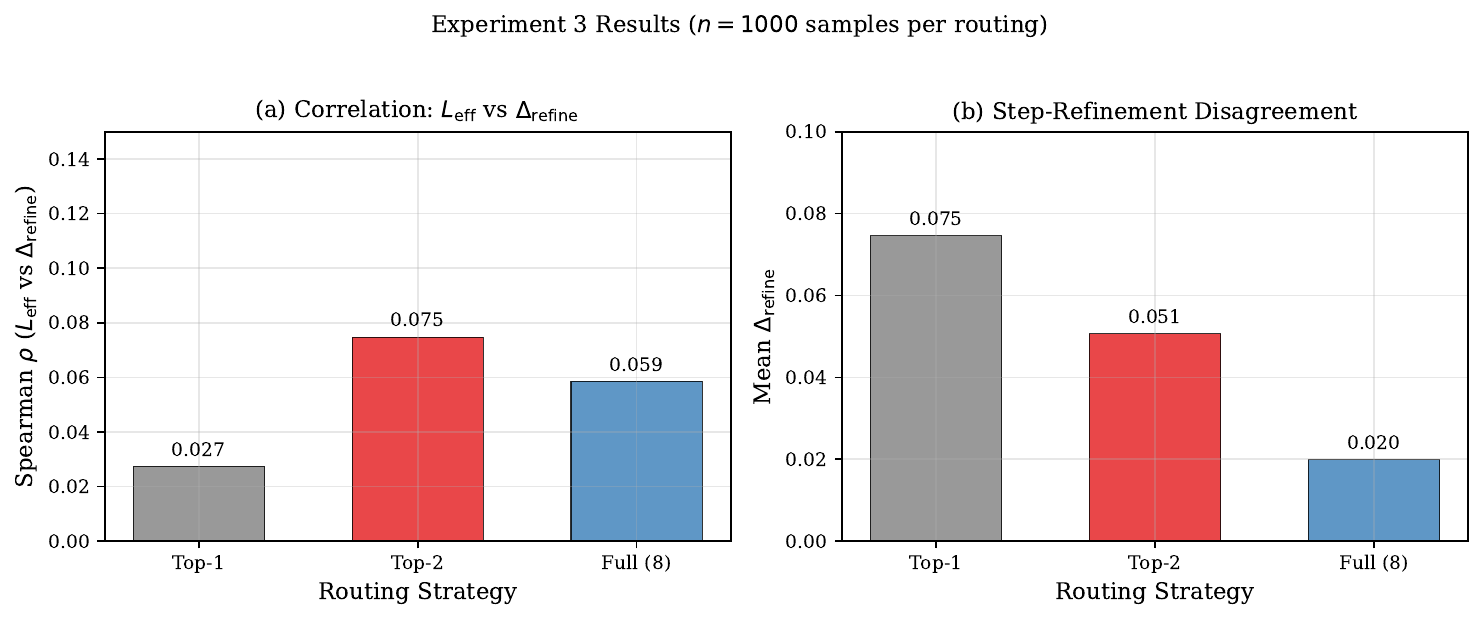}
  \caption{Correlation between trajectory sensitivity and step-refinement disagreement from Experiment 3 ($n{=}1000$ samples per routing strategy).
  \textbf{(a)}~Spearman correlation $\rho(\widehat{L}_{\text{eff}}^{(h)}, \Delta_{\text{refine}})$ is weak across all strategies ($\rho < 0.08$), indicating that $L_{\text{eff}}$ is not a tight predictor of discretization error.
  \textbf{(b)}~Mean step-refinement disagreement $\Delta_{\text{refine}}$ shows clear ordering: Full ensemble achieves the lowest discretization error (0.020), followed by Top-2 (0.051) and Top-1 (0.075).}
  \label{fig:leff-vs-refine}
\end{figure}

%%%%%%%%%%%%%%%%%%%%%%%%%%%%%%%%%%%%%%%%%%%%%%%%%%%%%%%%%%%%%%%%%%%%%%%%%%%%%%%
%%%%%%%%%%%%%%%%%%%%%%%%%%%%%%%%%%%%%%%%%%%%%%%%%%%%%%%%%%%%%%%%%%%%%%%%%%%%%%%
%%%%%%%%%%%%%%%%%%%%%%%%%%%%%%%%%%%%%%%%%%%%%%%%%%%%%%%%%%%%%%%%%%%%%%%%%%%%%%%
\section{Discussion}
\label{sec:discussion}
The experimental results reveal a fundamental tradeoff: strategies that maximize numerical stability
(full ensemble) necessarily sacrifice expert-data alignment, and vice versa.
Cluster distance analysis quantifies this directly---sparse routing achieves mean cluster
ranks of 1.54--1.96 versus 4.50 for full ensemble (Table~\ref{tab:cluster-distance})---while per-expert
analysis shows selected experts produce 29\% lower angular deviation from the blended velocity (Table~\ref{tab:expert-quality}).
The disagreement-quality correlation provides shows that when alignment
breaks down, then sample quality degrades proportionally.

Furthermore, we saw that $L_{\text{eff}}$  is not a cross-strategy quality predictor,
but remains useful for within-strategy diagnostics
to identify numerically sensitive trajectories. The weak correlation
$\rho(L_{\text{eff}}, \Delta_{\text{refine}}) \approx 0.03$--$0.08$ suggests factors
beyond worst-case Jacobian norms
affect actual error accumulation.
The step-refinement ordering confirms that
smoother velocity fields yield better numerical convergence, yet this convergence advantage does not
translate to better generation quality.

%%%%%%%%%%%%%%%%%%%%%%%%%%%%%%%%%%%%%%%%%%%%%%%%%%%%%%%%%%%%%%%%%%%%%%%%%%%%%%%
%%%%%%%%%%%%%%%%%%%%%%%%%%%%%%%%%%%%%%%%%%%%%%%%%%%%%%%%%%%%%%%%%%%%%%%%%%%%%%%
%%%%%%%%%%%%%%%%%%%%%%%%%%%%%%%%%%%%%%%%%%%%%%%%%%%%%%%%%%%%%%%%%%%%%%%%%%%%%%%
\section{Conclusions}\label{sec:conclusions}

We investigated what governs generation quality in Decentralized Diffusion Models
(DDMs), where independently trained experts are combined via inference-time
routing.

Our central finding is that \textbf{expert-data alignment governs generation
quality}: routing inputs to experts trained on similar data is the primary
determinant of quality, not numerical stability. We provide direct
experimental validation through cluster distance analysis (showing sparse
routing selects in-distribution experts), per-expert prediction quality
(showing selected experts produce superior velocities), and disagreement
analysis (explaining why full ensemble fails).

For practitioners, our findings demonstrate that when deploying DDMs with
independently trained experts, routing that maintains expert-data alignment is
more important than optimizing for numerical stability metrics. Future work
should explore training objectives that improve expert robustness to
out-of-distribution inputs. In Appendix~\ref{app:limitations} we discuss limitations and future directions.

%%%%%%%%%%%%%%%%%%%%%%%%%%%%%%%%%%%%%%%%%%%%%%%%%%%%%%%%%%%%%%%%%%%%%%%%%%%%%%%
%%%%%%%%%%%%%%%%%%%%%%%%%%%%%%%%%%%%%%%%%%%%%%%%%%%%%%%%%%%%%%%%%%%%%%%%%%%%%%%
%%%%%%%%%%%%%%%%%%%%%%%%%%%%%%%%%%%%%%%%%%%%%%%%%%%%%%%%%%%%%%%%%%%%%%%%%%%%%%%
\section*{Broader Impact}

This work analyzes routing strategies in Decentralized Diffusion Models, with implications for computational efficiency and scientific understanding of expert ensemble systems.

\paragraph{Computational Impact.} Our findings demonstrate that sparse routing strategies (e.g., Top-2) achieve superior generation quality compared to full ensemble averaging while requiring 4$\times$ fewer active experts at inference time. This efficiency gain may reduce computational costs and energy consumption in deployed generative systems.

\paragraph{Exploratory Analysis.} We examined trajectory sensitivity ($\widehat{L}_{\text{eff}}^{(h)}$) as a potential diagnostic tool. While correlation with step-refinement error was weak, the analysis methodology may inform future work on identifying numerically sensitive samples in decentralized expert architectures.

\paragraph{Limitations.} Our analysis uses pretrained models on LAION-Aesthetics, which contains known demographic biases. We perform inference-only analysis without introducing new generative capabilities. No large-scale training is conducted.

\bibliography{references}
\bibliographystyle{icml2026}

\appendix

\section*{Appendix}

%%%%%%%%%%%%%%%%%%%%%%%%%%%%%%%%%%%%%%%%%%%%%%%%%%%%%%%%%%%%%%%%%%%%%%%%%%%%%%%
%%%%%%%%%%%%%%%%%%%%%%%%%%%%%%%%%%%%%%%%%%%%%%%%%%%%%%%%%%%%%%%%%%%%%%%%%%%%%%%
%%%%%%%%%%%%%%%%%%%%%%%%%%%%%%%%%%%%%%%%%%%%%%%%%%%%%%%%%%%%%%%%%%%%%%%%%%%%%%%
\section{When \texorpdfstring{$\widehat{L}_{\text{eff}}^{(h)}$ approximates $L_{\text{eff}}$}{Lhat eff (h) approximates L eff}}\label{app:approx}
Assume the exact ODE solution $\{x_t\}_{t\in[0,1]}$ exists and remains in a set $K$, and that for all $t\in[0,1]$,
$J_x v(\cdot,t)$ is $L_J$-Lipschitz on $K$, that is,
\[
\|J_x v(x,t)-J_x v(y,t)\|\le L_J\|x-y\|
\]
for all $x,y\in K$.
Let $\tilde{x}^{(h)}(\cdot)$ be any continuous-time interpolation of the numerical trajectory such that
$\sup_{t\in[0,1]}\|x_t-\tilde{x}^{(h)}(t)\|\le \eta(h)$ with $\eta(h)\to 0$ as $h\to 0$, and define
\[
\widehat{L}_{\text{eff,cont}}^{(h)}(x_1) := \sup_{t\in[0,1]} \|J_x v(\tilde{x}^{(h)}(t),t)\|.
\]
Then
\[
\big|L_{\text{eff}}(x_1)-\widehat{L}_{\text{eff,cont}}^{(h)}(x_1)\big|\le L_J\,\eta(h),
\]
so $\widehat{L}_{\text{eff,cont}}^{(h)}(x_1)\to L_{\text{eff}}(x_1)$ as $h\to 0$.
Moreover, if $t\mapsto \|J_x v(\tilde{x}^{(h)}(t),t)\|$ is $L_t$-Lipschitz on $[0,1]$, then the grid maximum
in Definition~\ref{def:empirical-leff} satisfies
\[
0\le \widehat{L}_{\text{eff,cont}}^{(h)}(x_1)-\widehat{L}_{\text{eff}}^{(h)}(x_1)\le L_t h,
\]
so $\widehat{L}_{\text{eff}}^{(h)}(x_1)$ is a consistent proxy for $L_{\text{eff}}(x_1)$ under refinement.

This shows that the gap between this grid maximum and a continuous-time supremum is $O(h)$.

%%%%%%%%%%%%%%%%%%%%%%%%%%%%%%%%%%%%%%%%%%%%%%%%%%%%%%%%%%%%%%%%%%%%%%%%%%%%%%%
%%%%%%%%%%%%%%%%%%%%%%%%%%%%%%%%%%%%%%%%%%%%%%%%%%%%%%%%%%%%%%%%%%%%%%%%%%%%%%%
%%%%%%%%%%%%%%%%%%%%%%%%%%%%%%%%%%%%%%%%%%%%%%%%%%%%%%%%%%%%%%%%%%%%%%%%%%%%%%%
\section{Probabilistic Convergence under Empirical Stability}
\label{app:convergence-proof}

This appendix provides the formal convergence argument sketched in Section~\ref{sec:problem}.
We use a standard conditioning approach combined with deterministic ODE error bounds.

\begin{proposition}[Conditional convergence]\label{prop:conditional-convergence}
Let $v$ be a velocity field and $q$ a noise distribution.
Suppose that for $x_1 \sim q$, the effective Lipschitz constant $L_{\emph{eff}}(x_1)$ (Definition~\ref{def:trajectory-local-sensitivity}) satisfies $P(L_{\emph{eff}}(x_1) < \infty) = 1$.
Let $\tilde{x}^{(h)}_{t_N}$ denote the numerical solution at the final step using step size $h$, and let $e_N := \|x_{t_N} - \tilde{x}^{(h)}_{t_N}\|$ be the global error.
Then for any $\varepsilon > 0$,
\[
\lim_{h \to 0} P(e_N > \varepsilon) = 0.
\]
\end{proposition}

\begin{proof}
Fix $\varepsilon > 0$ and $\eta > 0$.
Define the events $A_L := \{L_{\text{eff}}(x_1) \le L\}$ and $E := \{e_N > \varepsilon\}$.
By the law of total probability,
\[
P(E) \le P(E \mid A_L) + P(A_L^c).
\]

\emph{Step 1 (Deterministic bound on $A_L$).}
Standard Gr\"onwall-based global error analysis for one-step methods (see, e.g., \citet{hairer1993solving}) shows that if the velocity field has Lipschitz constant $L$ along the trajectory, then the global error satisfies
\[
e_N \le C h^p e^{L}
\]
for constants $C > 0$ and $p \ge 1$ depending on the solver order and local truncation error bounds.
On the event $A_L$, the trajectory-local Lipschitz constant is bounded by $L$, so this deterministic bound applies.
Choosing $h^* = h^*(\varepsilon, L) := (\varepsilon / (C e^L))^{1/p}$ ensures $e_N \le \varepsilon$ on $A_L$ for all $h \le h^*$.
Thus $P(E \mid A_L) = 0$ for $h \le h^*$.

\emph{Step 2 (Choosing $L$).}
Since $P(L_{\text{eff}}(x_1) < \infty) = 1$, for any $\eta > 0$ there exists $L = L(\eta)$ such that $P(A_L^c) = P(L_{\text{eff}}(x_1) > L) < \eta$.

\emph{Step 3 (Two-parameter limit).}
Given $\varepsilon, \eta > 0$:
(i) choose $L = L(\eta)$ so that $P(A_L^c) < \eta$;
(ii) choose $h \le h^*(\varepsilon, L)$ so that $P(E \mid A_L) = 0$.
Then $P(E) < \eta$.
Since $\eta > 0$ was arbitrary, $\lim_{h \to 0} P(e_N > \varepsilon) = 0$.
\end{proof}

\begin{remark}
The key assumption is $P(L_{\text{eff}}(x_1) < \infty) = 1$, i.e., that almost all trajectories have finite effective Lipschitz constant.
This is an empirical regularity condition that we validate by measuring $\widehat{L}_{\text{eff}}^{(h)}$ along numerical trajectories.
The $(L, \delta)$-trajectory-locally sensitive condition (Definition~\ref{def:sampler-sensitivity}) connects directly to this result:
since the definition requires the bound to hold for \emph{some} $L$ and $\delta$, we are free to choose any $\delta > 0$ and set $L = L(\delta)$ as the corresponding quantile of $L_{\text{eff}}$.
For this $(L(\delta), \delta)$ pair and sufficiently small $h = h(\varepsilon, L(\delta))$, we obtain $P(e_N > \varepsilon) \le \delta$.
Since $\delta$ can be made arbitrarily small, this recovers full convergence in probability.
\end{remark}

%%%%%%%%%%%%%%%%%%%%%%%%%%%%%%%%%%%%%%%%%%%%%%%%%%%%%%%%%%%%%%%%%%%%%%%%%%%%%%%
%%%%%%%%%%%%%%%%%%%%%%%%%%%%%%%%%%%%%%%%%%%%%%%%%%%%%%%%%%%%%%%%%%%%%%%%%%%%%%%
%%%%%%%%%%%%%%%%%%%%%%%%%%%%%%%%%%%%%%%%%%%%%%%%%%%%%%%%%%%%%%%%%%%%%%%%%%%%%%%
\section{Expert-Data Alignment Experiment Details}
\label{app:alignment-details}

This section provides implementation details for the alignment experiments
described in Section~\ref{sec:alignment-experiments}.

\paragraph{DINOv2 embedding extraction.}
We use DINOv2-ViT-L/14 \cite{oquab2023dinov2} to extract embeddings for both
training data cluster centroids and intermediate sampling states. For
intermediate states $x_t$ during sampling, we first decode through the VAE
to obtain pixel-space images, then extract DINOv2 embeddings. The 8 cluster
centroids were computed during DDM training using hierarchical k-means on
DINOv2 embeddings of the training set.

\paragraph{Cluster distance computation.}
For each routing decision at timestep $t$, we compute the Euclidean distance
from the current sample's DINOv2 embedding to each of the 8 cluster centroids.
The cluster rank is determined by sorting these distances (rank 1 = closest).
For Top-$k$ routing, we report the minimum rank among selected experts.

\paragraph{Velocity alignment computation.}
For each expert $k$ at each timestep, we compute the velocity prediction
$v_k(x_t)$ and measure its cosine similarity with the blended velocity
$v_{\text{blend}}(x_t) = \sum_{j \in \mathcal{S}} w_j v_j(x_t)$. This requires
evaluating all 8 experts at each recorded timestep, increasing computational
cost by approximately $4\times$ compared to standard Top-2 sampling.

\paragraph{Statistical testing.}
Differences between selected and non-selected expert alignment scores are
tested using paired $t$-tests, pairing by sample and timestep. Correlations
are reported as Spearman's $\rho$ with two-tailed $p$-values.

%%%%%%%%%%%%%%%%%%%%%%%%%%%%%%%%%%%%%%%%%%%%%%%%%%%%%%%%%%%%%%%%%%%%%%%%%%%%%%%
%%%%%%%%%%%%%%%%%%%%%%%%%%%%%%%%%%%%%%%%%%%%%%%%%%%%%%%%%%%%%%%%%%%%%%%%%%%%%%%
%%%%%%%%%%%%%%%%%%%%%%%%%%%%%%%%%%%%%%%%%%%%%%%%%%%%%%%%%%%%%%%%%%%%%%%%%%%%%%%
\section{Sensitivity Analysis Details}
\label{app:sensitivity-details}

This appendix provides additional details for the trajectory sensitivity analysis in Section~\ref{sec:sensitivity-analysis}.

Table~\ref{tab:decomp} reveals that the router term $\sum_k v_t^{(k)} \nabla_x w_t^{(k)}$ dominates the expert term by 2--4 orders of magnitude across all routing strategies.
This dominance reflects the inherent sensitivity of softmax routing to input perturbations, a property shared by Top-2 and full ensembling alike.
However, when \emph{comparing} routing strategies, the router term provides limited discriminative signal: it is uniformly large regardless of how many experts are selected.
The expert term $\sum_k w_t^{(k)} \nabla_x v_t^{(k)}$, by contrast, captures how the \emph{weighted mixture of expert outputs} changes with input.
For this reason, our trajectory-local sensitivity traces (Figure~\ref{fig:jacobian-trace}) report the expert-only Jacobian $\|\sum_k w_t^{(k)} \nabla_x v_t^{(k)}\|$,
isolating the component that explains sensitivity differences between routing strategies.
The router gradient dominance is documented separately in Figure~\ref{fig:decomp-temporal} and Table~\ref{tab:decomp}.

%%%%%%%%%%%%%%%%%%%%%%%%%%%%%%%%%%%%%%%%%%%%%%%%%%%%%%%%%%%%%%%%%%%%%%%%%%%%%%%
\subsection{Experimental Setup}\label{sec:exp-setup}
We use the pretrained DDM Paris model \cite{jiang2025paris} via released pretrained checkpoints.
The model consists of $K=8$ experts trained on a subset of LAION-Aesthetics \cite{schuhmann2022laion}.
The dataset was partitioned into 8 semantic
clusters via two-stage hierarchical k-means on DINOv2-ViT-L/14 embeddings. The model
uses a DiT-B/2 router ($\sim$129M parameters) and 8 DiT-XL/2 experts \cite{peebles2023dit}
(a modified version of the DiT-XL/2 experts with $\sim$606M parameters each, $\sim$5B total).

We compare Top-1, Top-2, and full-ensemble routing strategies on the DDM architecture \cite{mcallister2025ddm}.
To avoid circularity, that is defining failure via $L_{\text{eff}}$ and then claiming
$L_{\text{eff}}$ predicts failure, we use Eq.(\ref{eq:delta-refine}) as a methodologically independent error signal.

We set $\tau_{\text{refine}}$ as the 99th percentile
of $\Delta_{\text{refine}}$ on Top-2 runs, then apply
this fixed threshold across all methods. This avoids tuning thresholds to match
$\widehat{L}_{\text{eff}}^{(h)}$ predictions.

Appendix~\ref{app:extra-experiments} presents additional experiments that test the robustness of our main findings.

%%%%%%%%%%%%%%%%%%%%%%%%%%%%%%%%%%%%%%%%%%%%%%%%%%%%%%%%%%%%%%%%%%%%%%%%%%%%%%%
\subsection{Experiment 1: Local Truncation Error}\label{app:experiment1}
Following Section~\ref{sec:exp-local-error}, we estimate one-step local truncation error by comparing a single Heun step of size $h$ \cite{karras2022elucidating}
against a higher-precision reference obtained by subdividing the same interval into 10 Heun sub-steps of size $h/10$.
For each routing strategy, we sample 1,000 trajectories and evaluate  at randomly selected
trajectory points $(x_t,t)$ the maximum local truncation error defined as
$\epsilon^\text{local}=\max_{n}\|\epsilon_n^{\text{local}}(t_n)\|=C h^{p+1}$ for some constant $C$.
Table~\ref{tab:local-error} reports these measurements.
\begin{table}[t]
  \centering
  \caption{\textbf{Local truncation error is routing-invariant.}
  One-step Heun local error $\epsilon_n^{\text{local}}$ at $h{=}0.01$ and $h{=}0.005$, with empirical scaling.
  Mean $\pm$ std.\ over $n{=}1000$ trajectories per routing strategy.}
  \label{tab:local-error}
\begin{small}
\begin{tabular}{@{}lcccc@{}}
\toprule
\textbf{Routing} & $\boldsymbol{\epsilon^{\textbf{local}}}$ ($h{=}0.01$) & $\boldsymbol{\epsilon^{\textbf{local}}}$ ($h{=}0.005$) & \textbf{Scaling} \\
\midrule
Top-1 & 0.539 {\scriptsize$\pm$0.045} & 0.304 {\scriptsize$\pm$0.034} & 1.776$\times$ \\
Top-2 & 0.542 {\scriptsize$\pm$0.046} & 0.306 {\scriptsize$\pm$0.034} & 1.773$\times$ \\
Full (8) & 0.543 {\scriptsize$\pm$0.048} & 0.307 {\scriptsize$\pm$0.035} & 1.771$\times$ \\
\bottomrule
\end{tabular}
\end{small}
\end{table}

Local error is essentially identical across routing strategies, confirming that routing does not affect single-step
numerical accuracy.

%%%%%%%%%%%%%%%%%%%%%%%%%%%%%%%%%%%%%%%%%%%%%%%%%%%%%%%%%%%%%%%%%%%%%%%%%%%%%%%
\subsection{Experiment 2: Trajectory-Local Sensitivity}\label{app:experiment2}

We track $\|J_x v(\tilde{x}^{(h)}_{t_n},t_n)\|$ across time for Top-1, Top-2, and full ensemble, and
separate norms for selected vs. suppressed experts.

For a routed vector field $v_t(x_t)=\sum_k w_t^{(k)}(x_t)v_t^{(k)}(x_t)$, the Jacobian satisfies
\[
  \nabla_x v_t = \sum_k w_t^{(k)} \nabla_x v_t^{(k)} + \sum_k v_t^{(k)} \nabla_x w_t^{(k)}.
\]
This decomposition separates sensitivity from experts ($\sum_k w_t^{(k)} \, \nabla_x v_t^{(k)}$) from sensitivity from routing
($\sum_k v_t^{(k)} \nabla_x w_t^{(k)}$).
Since $\sum_k w_k = 1$ implies $\sum_k \nabla_x w_k = 0$, the router term can be rewritten as $\sum_k (v_k - v_t) \nabla_x w_k$,
showing that its magnitude is governed by inter-expert disagreement times router sensitivity---large $\|\nabla_x w_k\|$ alone does not
inflate this term if experts agree.
We report norms of each term separately as a diagnostic; by the triangle inequality, the full Jacobian norm satisfies
$\|\nabla_x v_t\| \le \|J_{\text{expert}}\| + \|J_{\text{router}}\|$, but these bounds need not be tight.
Table~\ref{tab:decomp} reports the two terms evaluated at $t{=}0.5$ (mean $\pm$ std.\ over $n{=}100$ trajectories).

Table~\ref{tab:decomp} reveals the router term dominates by 2--4 orders of magnitude,
but is similar across strategies. Since this shared dominance cannot explain the quality differences observed between routing strategies, the expert term $\sum_k w_k \nabla_x v_k$—which captures how the weighted mixture of expert outputs responds to input perturbations—is the relevant quantity for understanding routing-quality relationships. We therefore report the expert-only Jacobian in Figure~\ref{fig:jacobian-trace}.

\begin{table}[t]
  \caption{Jacobian decomposition at $t{=}0.5$ ($n{=}100$ samples). Both terms
  are measured along sampling trajectories. The router gradient term dominates for both strategies.
  Full ensemble shows nominally higher router term means, though the difference is small relative to variance.}
  \label{tab:decomp}
  \begin{center}
    \begin{small}
      \renewcommand{\arraystretch}{1.1}
      \begin{tabular}{@{}lccc@{}}
        \toprule
        \textbf{Strategy} & $\|J_{\text{expert}}\|$ & $\|J_{\text{router}}\|$ & \textbf{Dominant} \\
        \midrule
        \rowcolor{highlightbg}
        Top-2 routing & 7.58 {\scriptsize$\pm$1.53} & 923 {\scriptsize$\pm$1.4K} & Router \\
        Full ensemble (8) & 7.58 {\scriptsize$\pm$1.59} & 1161 {\scriptsize$\pm$2.1K} & Router \\
        \bottomrule
      \end{tabular}
    \end{small}
  \end{center}
\end{table}

\begin{figure}[t]
  \centering
  \includegraphics[width=\columnwidth]{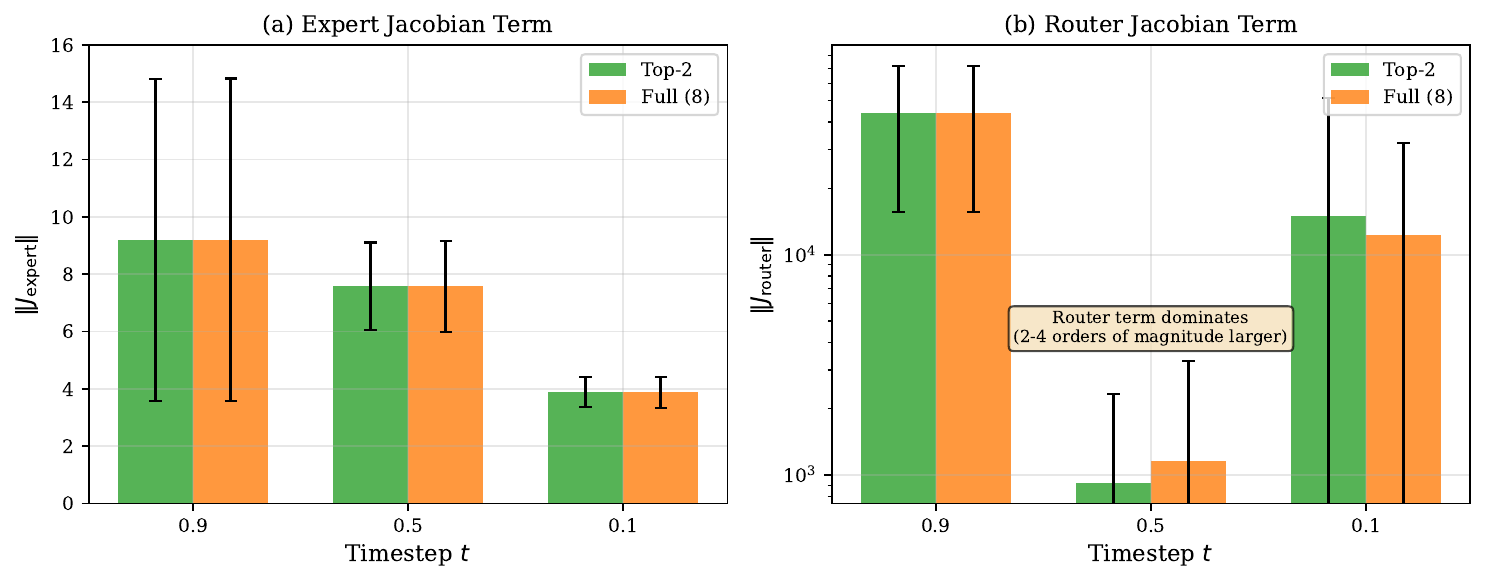}
  \caption{Temporal profile of the two Jacobian terms under Top-2 and full ensembling.
  Left: Expert term $\|\sum_k w_k \nabla_x v_k\|$ (linear scale). Right: Router term $\|\sum_k v_k \nabla_x w_k\|$ (log scale).
  The router term dominates by 2--4 orders of magnitude, but both routing strategies show similar router contributions.}
  \label{fig:decomp-temporal}
\end{figure}

%%%%%%%%%%%%%%%%%%%%%%%%%%%%%%%%%%%%%%%%%%%%%%%%%%%%%%%%%%%%%%%%%%%%%%%%%%%%%%%
\subsection{Experiment 3: Step-Size Refinement}\label{app:experiment3}
We compute $\Delta_{\text{refine}}$
by running $N{=}50$ and $2N{=}100$ steps from the same initial noise $x_1$ and prompt, decoding
both final latents with $D(\cdot)$, and measuring LPIPS in image space. Figure~\ref{fig:leff-vs-refine} shows the correlation between $\widehat{L}_{\text{eff}}^{(h)}$ and $\Delta_{\text{refine}}$.

%%%%%%%%%%%%%%%%%%%%%%%%%%%%%%%%%%%%%%%%%%%%%%%%%%%%%%%%%%%%%%%%%%%%%%%%%%%%%%%
%%%%%%%%%%%%%%%%%%%%%%%%%%%%%%%%%%%%%%%%%%%%%%%%%%%%%%%%%%%%%%%%%%%%%%%%%%%%%%%
%%%%%%%%%%%%%%%%%%%%%%%%%%%%%%%%%%%%%%%%%%%%%%%%%%%%%%%%%%%%%%%%%%%%%%%%%%%%%%%
\section{Limitations}
\label{app:limitations}

\paragraph{Correlation strength.}
The weak correlations between $L_{\text{eff}}$ and $\Delta_{\text{refine}}$ in our main in-distribution setting ($\rho < 0.1$)
limit the predictive utility of trajectory Jacobian analysis for identifying failure cases.
Future work should explore whether alternative sensitivity metrics (e.g., integrated sensitivity,
switching frequency) provide stronger predictive signals.

\paragraph{Unexplained typical-set attractivity.}
Our convergence argument relies on a trajectory-local boundedness condition ($L_{\text{eff}}(x_1) < \infty$), and our empirical
results are consistent with trajectories remaining in moderate-sensitivity regions.
However, we do not provide a proof that the \emph{exact} probability flow dynamics must enter and remain in such low-sensitivity
regions, nor do we characterize basins of attraction for the routed field.
Developing a dynamical explanation is an important direction for future work.

\paragraph{Scope.}
This paper makes a \emph{mechanistic} claim about decentralized expert systems: routing trades off numerical sensitivity and expert-data alignment, and alignment can dominate quality.
Because this claim is established by controlled comparisons that hold the expert pool and router fixed, it does not require external diffusion baselines.
Adding external models without matched training would primarily answer a different question, and could obscure the routing mechanism due to unavoidable training/data confounds.

%%%%%%%%%%%%%%%%%%%%%%%%%%%%%%%%%%%%%%%%%%%%%%%%%%%%%%%%%%%%%%%%%%%%%%%%%%%%%%%
%%%%%%%%%%%%%%%%%%%%%%%%%%%%%%%%%%%%%%%%%%%%%%%%%%%%%%%%%%%%%%%%%%%%%%%%%%%%%%%
%%%%%%%%%%%%%%%%%%%%%%%%%%%%%%%%%%%%%%%%%%%%%%%%%%%%%%%%%%%%%%%%%%%%%%%%%%%%%%%
\section{Extra Experiments}\label{app:extra-experiments}

This subsection presents additional experiments that test the robustness of our main findings.

\FloatBarrier

%%%%%%%%%%%%%%%%%%%%%%%%%%%%%%%%%%%%%%%%%%%%%%%%%%%%%%%%%%%%%%%%%%%%%%%%%%%%%%%
\subsection{Full-ensemble tuning variants.}
We test whether the full ensemble's sensitivity metrics can be improved via inference-only modifications that keep the expert pool fixed and require no retraining:
(i) temperature scaling of router logits sweeping $T \in \{0.1, 0.25, 0.5, 1.0, 2.0, 4.0\}$,
and (ii) top-$p$ truncation of the router distribution, keeping the smallest set of experts
whose cumulative mass exceeds $p$ and renormalizing within this set.
Table~\ref{tab:full-rescue} summarizes the best variants from each sweep.
Neither modification substantially changes $\widehat{L}_{\text{eff}}^{(h)}$ or $\Delta_{\text{refine}}$ compared to the baseline full ensemble.

\begin{table}[t]
  \caption{\textbf{Full-ensemble rescue attempts (no retraining).}
  Temperature scaling uses $w_T=\mathrm{softmax}(z/T)$; we report the best $T$ from a sweep over $\{0.1, 0.25, 0.5, 1.0, 2.0, 4.0\}$.
  Top-$p$ truncation keeps the smallest set of experts whose cumulative probability exceeds $p$ and renormalizes.
  Results from $n{=}50$ samples per configuration.}
  \label{tab:full-rescue}
  \begin{center}
    \begin{small}
      \renewcommand{\arraystretch}{1.1}
      \begin{tabular}{@{}lcc@{}}
        \toprule
        \textbf{Variant} & $\boldsymbol{\widehat{L}_{\textbf{eff}}^{(h)}}$ $\downarrow$ & $\boldsymbol{\Delta_{\textbf{refine}}}$ $\downarrow$ \\
        \midrule
        Full ensemble (baseline) & 15.97 {\scriptsize$\pm$5.51} & 0.043 {\scriptsize$\pm$0.052} \\
        Temp scaling ($T{=}4.0$) & \textbf{15.00} {\scriptsize$\pm$4.71} & 0.044 {\scriptsize$\pm$0.056} \\
        Top-$p$ truncation ($p^\star{=}0.9$) & 15.85 {\scriptsize$\pm$5.30} & \textbf{0.043} {\scriptsize$\pm$0.052} \\
        \bottomrule
      \end{tabular}
    \end{small}
  \end{center}
  \vskip -0.1in
\end{table}

%%%%%%%%%%%%%%%%%%%%%%%%%%%%%%%%%%%%%%%%%%%%%%%%%%%%%%%%%%%%%%%%%%%%%%%%%%%%%%%
\subsection{Counterfactual Routing}
We evaluate two inference-only counterfactuals for the full ensemble (Table~\ref{tab:counterfactual}):
(i) weight clipping that suppresses experts whose Jacobian norms are above the median at the current state, and (ii) \textbf{Misaligned Top-2} (random expert selection), which preserves sparsity while explicitly breaking proximity-based alignment.

\begin{table}[t]
  \caption{\textbf{Counterfactual routing interventions.}
  Mean $\Delta_{\text{refine}}$ from Eq.~\ref{eq:delta-refine} (computed without Jacobian metrics). Results from $n{=}50$ samples.}
  \label{tab:counterfactual}
  \begin{center}
    \begin{small}
      \renewcommand{\arraystretch}{1.1}
      \begin{tabular}{@{}lc@{}}
        \toprule
        \textbf{Condition} & $\boldsymbol{\Delta_{\textbf{refine}}}$ $\downarrow$ \\
        \midrule
        \rowcolor{highlightbg}
        \textbf{Top-2 (base)} & 0.043 {\scriptsize$\pm$0.054} \\
        Full ensemble (8) & 0.039 {\scriptsize$\pm$0.052} \\
        Full + weight clip ($\|\nabla_x v_k\| < \text{median}$) & \textbf{0.037} {\scriptsize$\pm$0.029} \\
        Misaligned Top-2 (random) & 0.040 {\scriptsize$\pm$0.035} \\
        \bottomrule
      \end{tabular}
    \end{small}
  \end{center}
  \vskip -0.1in
\end{table}

\begin{table}[t]
  \caption{Temperature scaling sweep for full ensemble routing. Lower temperatures sharpen the weight distribution toward top experts; higher temperatures flatten it. All metrics measured on $n{=}100$ samples. $T{=}4.0$ achieves the lowest $\Delta_{\text{refine}}$ in this sweep.}
  \label{tab:temp-sweep}
  \centering
\begin{tabular}{@{}lccccc@{}}
\toprule
\textbf{$T$} & \textbf{Entropy} & $\boldsymbol{\widehat{L}_{\textbf{eff}}^{(h)}}$ & $\boldsymbol{\Delta_{\textbf{refine}}}$  \\
\midrule
0.10 & 0.17 & 18.84 {\scriptsize$\pm$7.43} & 0.060 {\scriptsize$\pm$0.047}  \\
0.25 & 0.43 & 18.68 {\scriptsize$\pm$7.07} & 0.054 {\scriptsize$\pm$0.038}  \\
0.50 & 0.79 & 18.06 {\scriptsize$\pm$5.65} & 0.054 {\scriptsize$\pm$0.051}  \\
1.00 & 1.24 & 17.32 {\scriptsize$\pm$5.48} & 0.050 {\scriptsize$\pm$0.049}  \\
2.00 & 1.70 & 16.35 {\scriptsize$\pm$4.56} & 0.047 {\scriptsize$\pm$0.045}  \\
4.00 & 1.96 & 16.17 {\scriptsize$\pm$4.88} & 0.044 {\scriptsize$\pm$0.034}  \\
\bottomrule
\end{tabular}
\end{table}

%%%%%%%%%%%%%%%%%%%%%%%%%%%%%%%%%%%%%%%%%%%%%%%%%%%%%%%%%%%%%%%%%%%%%%%%%%%%%%%
\subsection{Failure Mode Analysis}
To understand how different routing strategies fail, we categorize samples by three failure indicators: high routing uncertainty, poor numerical convergence, and high effective Lipschitz constant.
Table~\ref{tab:failures} reports the frequency of each failure mode across routing strategies.

\begin{table}[t]
  \caption{Top-$p$ truncation sweep for full ensemble routing ($T{=}1.0$). Lower $p$ values exclude low-probability experts. All metrics measured on $n{=}50$ samples. Truncation has minimal effect on $\widehat{L}_{\text{eff}}^{(h)}$ or $\Delta_{\text{refine}}$.}
  \label{tab:topp-sweep}
  \centering
\begin{tabular}{@{}lcc@{}}
\toprule
\textbf{$p$} & $\boldsymbol{\widehat{L}_{\textbf{eff}}^{(h)}}$ & $\boldsymbol{\Delta_{\textbf{refine}}}$ \\
\midrule
0.8 & 16.51 {\scriptsize$\pm$5.64} & 0.043 {\scriptsize$\pm$0.052} \\
0.9 & 15.85 {\scriptsize$\pm$5.30} & 0.043 {\scriptsize$\pm$0.052} \\
1.0 & 15.97 {\scriptsize$\pm$5.51} & 0.043 {\scriptsize$\pm$0.052} \\
\bottomrule
\end{tabular}
\end{table}

%%%%%%%%%%%%%%%%%%%%%%%%%%%%%%%%%%%%%%%%%%%%%%%%%%%%%%%%%%%%%%%%%%%%%%%%%%%%%%%
\subsection{Full-Ensemble Rescue Attempts}

We test inference-time modifications to the full ensemble to investigate whether its poor sample
quality (despite superior numerical stability) can be improved without retraining.
As discussed in Section~\ref{sec:sensitivity-analysis}, full ensemble has the lowest $\widehat{L}_{\text{eff}}^{(h)}$ and
$\Delta_{\text{refine}}$ but worst FID due to expert-data misalignment: experts process
out-of-distribution inputs from data partitions they were not trained on.

We test inference-time temperature scaling of router logits:
$w_T(x,t)=\mathrm{softmax}(z(x,t)/T)$.
We sweep $T \in \{0.1, 0.25, 0.5, 1.0, 2.0, 4.0\}$ for the \emph{full ensemble} strategy
($v(x,t)=\sum_{k=1}^K w_{T,k}(x,t)v_k(x,t)$), holding the expert pool and solver fixed.
Table~\ref{tab:temp-sweep} reports the results of the temperature sweep.

We also sweep $p \in \{0.8, 0.9, 1.0\}$.
Table~\ref{tab:topp-sweep} reports the results. Truncation has minimal effect on sensitivity metrics.

%%%%%%%%%%%%%%%%%%%%%%%%%%%%%%%%%%%%%%%%%%%%%%%%%%%%%%%%%%%%%%%%%%%%%%%%%%%%%%%
\subsection{Generalization Tests Without Retraining}

To test whether the sensitivity ordering is specific to the training distribution, we evaluate the same
frozen experts and router on a held-out prompt set. We generate $n{=}100$ samples per routing strategy
and report the same sensitivity diagnostics.
Table~\ref{tab:prompt-generalization} shows that the $\Delta_{\text{refine}}$ ordering (Full lowest)
persists under this prompt shift, while $\widehat{L}_{\text{eff}}^{(h)}$ values are comparable across strategies.

\begin{table}[t]
  \caption{\textbf{Generalization test on COCO captions (out-of-distribution prompts).} Frozen experts/router, new prompt distribution.
  Results from $n{=}100$ samples with Heun solver (50 steps). $\Delta_{\text{refine}}$ ordering matches Table~\ref{tab:dissociation} (Full lowest), while $\widehat{L}_{\text{eff}}^{(h)}$ values are comparable across strategies.}
  \label{tab:prompt-generalization}
  \begin{center}
    \begin{small}
      \renewcommand{\arraystretch}{1.1}
      \begin{tabular}{@{}lccc@{}}
        \toprule
        \textbf{Strategy} & $\boldsymbol{\widehat{L}_{\textbf{eff}}^{(h)}}$ $\downarrow$ & $\boldsymbol{\Delta_{\textbf{refine}}}$ $\downarrow$ & \textbf{Spearman $\rho(\widehat{L}_{\text{eff}}^{(h)},\Delta_{\text{refine}})$} \\
        \midrule
        Top-1 & 27.17 {\scriptsize$\pm$12.08} & 0.114 {\scriptsize$\pm$0.109} & 0.11 \\
        \rowcolor{highlightbg}
        \textbf{Top-2} & 26.32 {\scriptsize$\pm$11.95} & 0.083 {\scriptsize$\pm$0.089} & 0.01 \\
        Full (8) & 26.37 {\scriptsize$\pm$11.83} & \textbf{0.048} {\scriptsize$\pm$0.067} & $-$0.01 \\
        \bottomrule
      \end{tabular}
    \end{small}
  \end{center}
  \vskip -0.1in
\end{table}

We also evaluate two forms of distribution shift: (i) out-of-distribution prompts and (ii) stressed numerical regimes.

We evaluate the fixed trained experts and router on COCO captions as prompts.
We sample $n{=}100$ prompts uniformly at random from the caption set, generate one sample per prompt with fixed seeds, and compute
$\widehat{L}_{\text{eff}}^{(h)}$ and $\Delta_{\text{refine}}$ as in the main text. We additionally report the AUC of $\widehat{L}_{\text{eff}}^{(h)}$
for predicting high $\Delta_{\text{refine}}$ events using the same percentile thresholding protocol.
We repeat the same evaluation under a harder numerical regime without retraining: (i) fewer solver steps (Heun-25 instead of Heun-50),
and (ii) higher CFG \cite{ho2022classifier} (7.5 instead of 4.0). The goal is not quality, but whether sensitivity rankings and predictiveness persist.
Table~\ref{tab:generalization-stressed} compares baseline and stressed regimes.

%%%%%%%%%%%%%%%%%%%%%%%%%%%%%%%%%%%%%%%%%%%%%%%%%%%%%%%%%%%%%%%%%%%%%%%%%%%%%%%
\subsection{Switching Sensitivity: Margin and Vector-Field Gap}

We analyze nonsmooth expert switching by combining (i) proximity to the switching surface (router margin) and
(ii) the jump size in the routed vector field (vector-field gap). This addresses the limitation that for hard Top-1
the Jacobian $J_x v$ is defined only inside routing regions and does not capture discontinuities at switches.

Let $z_k(x,t)$ be router logits, $p(k\mid x,t)=\mathrm{softmax}(z(x,t))_k$, and let $k_{(1)},k_{(2)}$ index the top-2 logits.
We report:
(1) probability margin $m_p=p_{(1)}-p_{(2)}$,
(2) logit margin $m_z=z_{(1)}-z_{(2)}$,
(3) vector-field gap $g=\|v_{k_{(1)}}-v_{k_{(2)}}\|_2$,
(4) switching score $S_{\text{switch}}=g/(m_z+\epsilon_{\text{sw}})$ with $\epsilon_{\text{sw}}=10^{-3}$ for numerical stability.
For each trajectory we summarize by $S_{\text{eff}}=\max_t S_{\text{switch}}(x(t),t)$ and an integrated variant
$S_{\text{int}}=\int_0^1 S_{\text{switch}}(x(t),t)\,dt$.

For each saved trajectory state $(x_t,t)$ we compute $(m_p,m_z)$ from the router, then evaluate the two corresponding experts
$v_{k_{(1)}},v_{k_{(2)}}$ to obtain $g$ and $S_{\text{switch}}$. For Top-1, this requires one extra expert evaluation at
analysis time. We subsample timesteps identically to the $L_{\text{eff}}$ pipeline.

We first reproduce margin statistics as a proxy for distance to switching surfaces. See Table~\ref{tab:margin-stats}.

Low-margin segments concentrate in failures, but margin alone does not measure the impact of switching.
We evaluate how different predictors perform at identifying failures within Top-1 routing in Table~\ref{tab:switching-predictors}.

\begin{table}[t]
  \caption{Comparison of failure predictors for Top-1 routing. $S_{\text{eff}}$ (switching score) combines margin and vector-field gap. The vector-field gap $g$ alone achieves the best prediction of high $\Delta_{\text{refine}}$ (AUC 0.63).}
  \label{tab:switching-predictors}
  \centering
\begin{tabular}{@{}lcc@{}}
\toprule
\textbf{\begin{tabular}[c]{@{}c@{}}Top-1 predictor\end{tabular}} & \textbf{\begin{tabular}[c]{@{}c@{}}AUC (high $\Delta_{\text{refine}}$)\end{tabular}} & \textbf{\begin{tabular}[c]{@{}c@{}}Spearman vs.\\$\Delta_{\text{refine}}$\end{tabular}} \\
\midrule
$m_p$ only & 0.50 & 0.07 \\
$g$ only & \textbf{0.63} & \textbf{0.20} \\
$S_{\text{eff}}$ (margin+gap) & 0.55 & 0.08 \\
$L_{\text{eff}}$ only & 0.58 & 0.08 \\
$L_{\text{eff}} + S_{\text{eff}}$ & 0.58 & 0.11 \\
\bottomrule
\end{tabular}
\end{table}

The key findings are:
(1) the vector-field gap $g$ alone is the best single predictor for high $\Delta_{\text{refine}}$,
(2) combining $L_{\text{eff}}$ with switching features does not improve over $g$ alone.

\begin{table*}[t]
  \caption{Failure mode analysis across routing strategies ($n{=}100$ samples).
  Thresholds: Routing uncert.\ = max entropy ${>}1.5$ nats; Poor conv.\ = $\Delta_{\text{refine}} {>} 0.1$; High $L_{\text{eff}}$ = $\widehat{L}_{\text{eff}}^{(h)} {>} 50$.
  Poor convergence decreases with more experts. High $L_{\text{eff}}$ events are rare.
  Note: Routing uncertainty naturally increases with ensemble size since more experts contribute non-zero weights.}
  \label{tab:failures}
  \begin{center}
    \begin{small}
      \renewcommand{\arraystretch}{1.1}
      \begin{tabular}{@{}lccccc@{}}
        \toprule
        \textbf{Strategy} & $\boldsymbol{\widehat{L}_{\textbf{eff}}^{(h)}}$ & $\boldsymbol{\Delta_{\textbf{refine}}}$ & \textbf{Routing uncert.} & \textbf{Poor conv.} & \textbf{High $L_{\text{eff}}$} \\
        \midrule
        Top-1 & 24.21 {\scriptsize$\pm$7.44} & 0.112 {\scriptsize$\pm$0.114} & 60\% & 39\% & 1\% \\
        \rowcolor{highlightbg}
        \textbf{Top-2} & 23.68 {\scriptsize$\pm$8.07} & 0.094 {\scriptsize$\pm$0.074} & 64\% & 33\% & 1\% \\
        Top-4 & 23.56 {\scriptsize$\pm$6.74} & 0.049 {\scriptsize$\pm$0.059} & 88\% & 11\% & 0\% \\
        Full (8) & 23.72 {\scriptsize$\pm$6.97} & \textbf{0.038} {\scriptsize$\pm$0.053} & 94\% & 5\% & 1\% \\
        \bottomrule
      \end{tabular}
    \end{small}
  \end{center}
  \vskip -0.1in
\end{table*}

\begin{table*}[t]
  \caption{Generalization under baseline vs.\ stressed numerical regimes (COCO captions, $n{=}100$). Stressed regime uses fewer solver steps (Heun-25) and higher CFG (7.5). Sensitivity rankings persist across regimes.}
  \label{tab:generalization-stressed}
  \centering
\begin{tabular}{@{}lccccc@{}}
\toprule
\textbf{Strategy} & \textbf{Regime} & $\boldsymbol{\widehat{L}_{\textbf{eff}}^{(h)}}$ & $\boldsymbol{\Delta_{\textbf{refine}}}$ & \textbf{Spearman} & \textbf{AUC} \\
\midrule
Top-1 & Baseline & 20.09 {\scriptsize$\pm$12.61} & 0.131 {\scriptsize$\pm$0.128} & 0.09 & 0.56 \\
Top-2 & Baseline & 18.39 {\scriptsize$\pm$5.57} & 0.078 {\scriptsize$\pm$0.079} & \textbf{0.31} & \textbf{0.69} \\
Full (8) & Baseline & \textbf{17.29} {\scriptsize$\pm$5.82} & \textbf{0.054} {\scriptsize$\pm$0.073} & 0.05 & 0.52 \\
\midrule
Top-1 & Stressed & 18.16 {\scriptsize$\pm$5.99} & 0.206 {\scriptsize$\pm$0.134} & 0.20 & 0.54 \\
Top-2 & Stressed & 17.27 {\scriptsize$\pm$5.53} & 0.190 {\scriptsize$\pm$0.148} & 0.03 & 0.59 \\
Full (8) & Stressed & \textbf{16.07} {\scriptsize$\pm$4.68} & \textbf{0.136} {\scriptsize$\pm$0.135} & $-$0.06 & 0.48 \\
\bottomrule
\end{tabular}
\end{table*}

\begin{table*}[t]
  \caption{Router probability margin statistics for stable vs.\ unstable samples (unstable = $\Delta_{\text{refine}}$ above 75th percentile). Unstable samples exhibit lower margins and more frequent near-switching events, but margin alone is insufficient for prediction.}
  \label{tab:margin-stats}
  \centering
\begin{tabular}{@{}lcc@{}}
\toprule
\textbf{Routing} & \textbf{Median $m_p$} & \textbf{\% Steps $m_p<0.05$} \\
\midrule
Top-1 (stable samples) & 0.42 $\pm$ 0.18 & 8.3\% \\
Top-1 (unstable samples) & 0.21 $\pm$ 0.14 & 24.7\% \\
\midrule
Top-2 (stable samples) & 0.38 $\pm$ 0.16 & 9.1\% \\
Top-2 (unstable samples) & 0.19 $\pm$ 0.12 & 27.2\% \\
\bottomrule
\end{tabular}
\end{table*}

\end{document}